\documentclass{article}
\usepackage{arxiv}
\usepackage[utf8]{inputenc}
\usepackage[T1]{fontenc}
\usepackage[authoryear,round]{natbib}
\setcitestyle{citesep={;},aysep={,},yysep={;}}
\renewcommand{\shorttitle}{When Better Gets Worse}

\usepackage{amsmath,amssymb,amsthm}
\usepackage{booktabs}
\usepackage{graphicx}
\usepackage{microtype}
\usepackage{xcolor}
\usepackage{url}
\usepackage{float}
\usepackage{algorithm}
\usepackage{algorithmic}
\usepackage{flafter}
\usepackage{placeins}
\usepackage[hidelinks]{hyperref}

\graphicspath{{figures/}}

\newcommand{\E}{\mathbb{E}}
\newcommand{\R}{\mathbb{R}}
\newcommand{\M}{\mathcal{M}}

\newcommand{\pii}{\pi}
\newcommand{\pip}{\pi'}
\newtheorem{proposition}{Proposition}
\newcommand{\OperatorShiftSeeds}{30}

\newcommand{\OperatorShiftEffect}{0.2627}
\newcommand{\OperatorShiftEffectCI}{[0.1684,\,0.3763]}

\newcommand{\ClosedLoopRounds}{40}

\newcommand{\ClosedLoopEffect}{1.4883}
\newcommand{\ClosedLoopEffectCI}{[1.0618,\,1.9105]}

\newcommand{\OODGain}{-0.2590}
\newcommand{\OODGainCI}{[-0.3588,\,-0.1707]}

\newcommand{\StrategicEffect}{-0.0240}
\newcommand{\StrategicEffectCI}{[-0.0249,\,-0.0231]}
\newcommand{\StrategicReversalRate}{0.9495}

\newcommand{\RevisionDecisionRoots}{90}
\newcommand{\RevisionDecisionFlips}{51}

\newcommand{\RevisionCriticalQueryRate}{0.97}
\newcommand{\RevisionCriticalQueryRateCI}{[0.92, 1.00]}
\newcommand{\RevisionCriticalQueryHitGain}{0.113}
\newcommand{\RevisionCriticalQueryHitGainCI}{[0.055, 0.180]}
\newcommand{\RevisionCriticalQueryMissGain}{-0.195}
\newcommand{\RevisionCriticalQueryMissGainCI}{[-0.584, 0.000]}
\newcommand{\RevisionCriticalNoiseMargin}{13.99}
\newcommand{\RevisionCriticalNoiseMarginCI}{[7.97, 22.90]}

\newcommand{\RevisionLeducCriticalQueryRate}{1.00}
\newcommand{\RevisionKuhnCriticalQueryRate}{1.00}
\newcommand{\RevisionMeltingCriticalQueryRate}{0.90}

\newcommand{\HighwayRedesignBudgetFourContrast}{$0.0380$}
\newcommand{\HighwayRedesignBudgetFourCI}{$[0.0300,0.0474]$}

\title{When Better Gets Worse: Improvement Fidelity\\
for Self-Improving Agents in Adaptive Worlds}
\author{
\makebox[0.43\textwidth]{Ke Wang}\\[0.2em]
\normalfont University of Cambridge\\
\normalfont Georgia Institute of Technology\\
\normalfont\small\href{mailto:colinwang@gatech.edu}{colinwang@gatech.edu}
\And
\makebox[0.43\textwidth]{Zijie Zhao\thanks{%
Corresponding author: Zijie Zhao (\href{mailto:zijiezha@mit.edu}{zijiezha@mit.edu}).}}\\[0.2em]
\normalfont Massachusetts Institute of Technology\\
\normalfont\small\href{mailto:zijiezha@mit.edu}{zijiezha@mit.edu}
\AND
\makebox[0.43\textwidth]{Zhiyi Yuan}\\[0.2em]
\normalfont University of Hong Kong\\
\normalfont\small\href{mailto:u3629247@connect.hku.hk}{u3629247@connect.hku.hk}
\And
\makebox[0.43\textwidth]{Changlun Li}\\[0.2em]
\normalfont The Hong Kong University of\\
\normalfont Science and Technology (Guangzhou)\\
\normalfont\small\href{mailto:cli942@connect.hkust-gz.edu.cn}{cli942@connect.hkust-gz.edu.cn}
}

\hypersetup{
  pdftitle={When Better Gets Worse: Improvement Fidelity for Self-Improving Agents in Adaptive Worlds},
  pdfauthor={Ke Wang; Zijie Zhao; Zhiyi Yuan; Changlun Li}
}
\date{}

\begin{document}
\maketitle
\setcounter{footnote}{0}

\begin{abstract}
Self-improving agents increasingly rely on proxy verifiers to choose policy
updates, yet deployment can change the world in which those updates are
evaluated. An update that looks better to the verifier can therefore become
worse after deployment even when the verifier ranks policies well overall.
We formalize this gap as \emph{Improvement Fidelity}, which asks whether proxy
improvements preserve the sign and ordering of deployment improvements over
the updates an improvement process actually proposes. We show that global
policy accuracy need not guarantee update fidelity: operator shift and
deployment response can create update-level errors, while candidate margins
determine whether those errors change the replacement decision. We introduce
\textbf{PIVOT-KG}, a paired, decision-aware validator that allocates scarce
high-fidelity evaluation according to the expected reduction in selection
regret per unit cost. Across 90 held-out roots in Leduc, Kuhn, and Melting Pot,
proxy and deployment optimal sets are disjoint in 51 cases. In an
eight-candidate HighwayEnv stress test, PIVOT-KG reduces mean
improvement-selection regret from 0.0435 under the exact Uniform validation
rule to 0.0055 at the primary budget. Together, these results show why reliable
self-improvement should evaluate proposed improvements in the worlds they
induce, while providing a practical rule for allocating scarce deployment
evidence when it can affect the replacement decision.\footnote{Code and
reproducibility materials:
{\urlstyle{same}\url{https://github.com/computational-decision-lab/pivot}}.}
\end{abstract}

\section{Introduction}

Self-improving agents repeatedly propose candidate policies, score them with a
verifier, and deploy a replacement. Yet an update that looks better to the
verifier can become worse after deployment when the deployed policy changes the
world in which it is evaluated. A lane change alters nearby trajectories, and
an action in a multi-agent game changes how other agents respond. A verifier
can therefore be accurate over policies in aggregate and still choose the wrong
local replacement. Figure~\ref{fig:reversal} illustrates this failure mode.

We call this transition-level requirement \emph{Improvement Fidelity}: whether
an improvement under the verifier remains an improvement after deployment in
the world induced by the update. Because deployment evaluation is often too
costly for every candidate, we introduce \textbf{PIVOT-KG}, the knowledge-gradient
instantiation of PIVOT (Paired Interventional Validation of Optimization
Transitions), a paired, decision-aware validator that directs a limited
high-fidelity (HF) budget to uncertainty that can change the replacement
decision.

Prior work separately studies deployment-induced response, policy-level
evaluation or safe improvement, and verifier reliability. We instead study an
operator-relative criterion for whether a proposed replacement remains an
improvement in the world it induces, together with a decision-aware rule for
allocating scarce deployment evidence for that decision.

\begin{figure}[H]
  \centering
  \includegraphics[width=\linewidth]{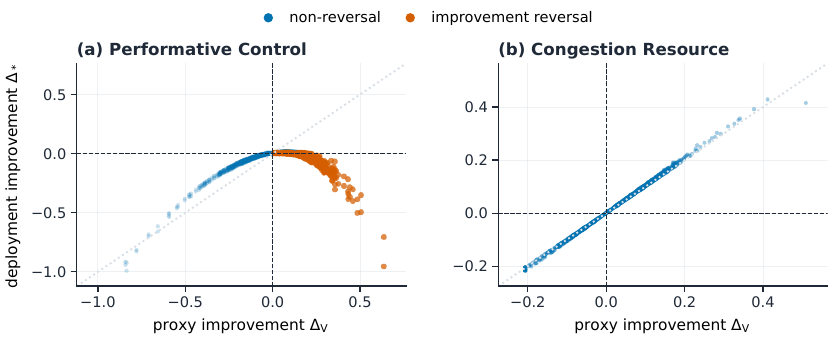}
  \caption{Proxy improvements can reverse after deployment. Each point is a
  proposed update; lower-right points are proxy-positive but
  deployment-negative.}
  \label{fig:reversal}
\end{figure}

Our contributions are:

\noindent\textbf{1. Improvement Fidelity.}
We formulate reliable self-improvement as an operator-relative, transition-level
evaluation problem and introduce Improvement Fidelity, showing why strong global
policy evaluation need not imply reliable evaluation of the updates an agent
actually proposes.

\smallskip
\noindent\textbf{2. Decision-aware validation.}
We introduce PIVOT-KG, which allocates a limited HF evaluation budget by its
expected reduction in update-selection regret rather than global prediction
error.

\smallskip
\noindent\textbf{3. Evidence across adaptive worlds.}
Across controlled response worlds, multi-agent environments, and reactive
driving simulations, we test the failure mechanisms predicted by the
theory and identify when targeted HF validation reduces selection regret.

\smallskip

Together, these contributions shift the evaluation of self-improvement from
whether policies are evaluated accurately in a proxy world to whether proposed
improvements remain improvements after deployment, while providing a principled
basis for allocating deployment evidence when it can affect the replacement decision.

\section{Related Work}

\paragraph{Performative and strategic response.}
Performative prediction and reinforcement learning study settings in which
deployment changes the data distribution, dynamics, rewards, or behavior of
other agents
~\citep{perdomo2020performative,mandal2023performative,rank2024performative,
mandal2025linear,gois2025performative,multiagentperformative2022}.
These works study how deployment-induced response affects learning, stability,
and optimality. We instead ask a local evaluation question: whether a
particular proposed replacement remains an improvement under the response
world it induces.

\paragraph{Policy evaluation and safe improvement.}
Off-policy evaluation and safe policy improvement estimate or certify the
performance of a candidate relative to a baseline
~\citep{jiang2016doubly,thomas2015hcpi,laroche2019spibb}, while
work on learned simulators studies policy evaluation, policy-value gaps,
or ranking preservation~\citep{worldgym2025,policyaware2026}.
Our focus is the proposed transition itself: a proxy can preserve policy-level
accuracy or rankings yet still induce the wrong local replacement decision
once deployment changes the evaluation world.

\paragraph{Self-improvement and verification.}
Recent work has made iterative policy or agent improvement operational
~\citep{evopolicygym2026,dgm2025} and has documented regressions or failures of
verification during self-improvement
~\citep{wu2024progress,selfauthored2026}. Related work on reward
overoptimization shows that optimizing an imperfect proxy can itself degrade
true performance~\citep{gao2023reward,laidlaw2025correlated}.
Our failure mode is complementary: the verifier need not be self-authored,
changing, or directly exploited. Even a fixed external verifier can select the
wrong update when deployment changes the world in which that update is
evaluated.

\paragraph{Adaptive selection and paired evaluation.}
Knowledge-gradient, ranking-and-selection, and best-arm methods allocate costly
observations to decision-relevant alternatives under limited budgets
~\citep{chick2001crn,frazier2009kg,kalyanakrishnan2012pac,xie2016pairwise}.
PIVOT-KG applies this decision-aware allocation principle to
proxy--deployment validation, modeling the correction in improvement and
valuing HF queries by their effect on the replacement decision.

The gap we study is therefore operator-relative and update-level: whether the
specific replacements an improvement process actually proposes preserve the
sign and ordering of improvement under the worlds induced by their deployment,
and how scarce deployment evidence should be allocated when that uncertainty can
change the replacement decision.

\section{Improvement Fidelity: Problem Formulation}

\paragraph{Proxy and deployment improvement.}
Let $\M[\pi]$ denote the deployment world induced by policy $\pi$, including
the environment response to that policy. Let $V$ denote the cheap verifier,
with value $J_V(\pi)$, and let $J(\pi;\M[\pi])$ denote the deployment value
of $\pi$ in its induced world. For a
proposed replacement $\tau=(\pii,\pip)$, define
\begin{equation}
  \Delta_V=J_V(\pip)-J_V(\pii),
  \qquad
  \Delta_*=J(\pip;\M[\pip])-J(\pii;\M[\pii]).
  \label{eq:improvement}
\end{equation}
Thus $\Delta_V$ is the verifier-estimated improvement and $\Delta_*$ is the
deployment improvement.

\paragraph{Operator-relative Improvement Fidelity.}
Let an improvement operator $\mathcal A$ induce a transition law
$Q_{\mathcal A}$ over proposed replacements. For a transition-level loss $L$,
we quantify the corresponding Improvement Fidelity error by
\begin{equation}
  \mathcal E_{\mathrm{IF}}(V,\mathcal A;L)
  =
  \E_{\tau\sim Q_{\mathcal A}}
  [L(\Delta_V,\Delta_*)].
  \label{eq:if}
\end{equation}
The relevant population is the updates the improvement process actually
proposes rather than the policy space in general; smaller values indicate
greater fidelity.

\paragraph{Fidelity and decision metrics.}
Improvement Differential Error (IDE) is the absolute-error instance of
Eq.~\eqref{eq:if}. We define
$\mathrm{IDE}=\E_{Q_{\mathcal A}}[|\Delta_V-\Delta_*|]$ for error magnitude,
$\mathrm{ISC}=\Pr_{Q_{\mathcal A}}(\operatorname{sgn}\Delta_V=
\operatorname{sgn}\Delta_*\mid\Delta_V\Delta_*\neq0)$ for sign agreement
after removing ties, and
$\mathrm{IRR}=\Pr_{Q_{\mathcal A}}(\Delta_*<0\mid\Delta_V>0)$ for the
proxy-positive reversal rate. Empirical results report the corresponding
sample estimates. Let $C=\{1,\ldots,K\}$ index candidate updates and
$C^+=\{0\}\cup C$ include the no-update action, with
$\Delta_{V,0}=\Delta_{*,0}=0$. For a selected action $\hat j\in C^+$, define
improvement-selection regret (ISR) as
\begin{equation}
  \mathrm{ISR}(\hat j)
  =
  \max_{j\in C^+}\Delta_{*,j}-\Delta_{*,\hat j}.
  \label{eq:isr}
\end{equation}
IDE, ISC, and IRR diagnose transition-level fidelity; ISR measures its
consequence for the replacement decision.

\paragraph{Deployment response levels.}
To separate sources of deployment response, we consider the following levels.
In the strategic case, $\pi$ denotes the focal agent $i$'s policy and
$R_{-i}(\pi)$ the response of the other agents:
\begin{align*}
  \Delta_{\mathrm{direct}}
  &=J(\pip;\M[\pii])-J(\pii;\M[\pii]),\\
  \Delta_{\mathrm{actor}}
  &=J(\pip;\M[\pip])-J(\pii;\M[\pii]),\\
  \Delta_{\mathrm{strategic}}
  &=J_i(\pip,R_{-i}(\pip))-J_i(\pii,R_{-i}(\pii)).
\end{align*}
The direct level holds the incumbent world fixed, the actor level allows the
environment to respond to the candidate, and the strategic level additionally
allows other agents to adapt through a response map $R_{-i}$. Here $J_i$
denotes realized value after both environment and other-agent response; best
response is one special case. In settings without other-agent adaptation,
$\Delta_{\mathrm{actor}}=\Delta_*$.

\section{Properties of Improvement Fidelity}

The results below establish when policy-level evaluation controls update-level
fidelity, how that connection can fail, and when the resulting error changes a
replacement decision.

\paragraph{From policy fidelity to update fidelity.}
\begin{proposition}[Uniform value fidelity is sufficient]
If $\sup_\pi|J_V(\pi)-J(\pi;\M[\pi])|\leq\varepsilon$, then
$|\Delta_V-\Delta_*|\leq2\varepsilon$. The sign is preserved whenever
$|\Delta_*|>2\varepsilon$.
\end{proposition}
\begin{proof}
$\Delta_V-\Delta_*=
[J_V(\pip)-J(\pip;\M[\pip])]-[J_V(\pii)-J(\pii;\M[\pii])]$; the triangle
inequality gives the bound and the sign statement.
\end{proof}

\begin{proposition}[Global policy accuracy need not imply update fidelity]
For every $\varepsilon>0$, there exist a finite policy family, a verifier,
a policy distribution, and an improvement operator with policy mean absolute error (MAE) and
Spearman rank error below $\varepsilon$, but $\mathrm{IRR}=1$ and
$\mathrm{ISC}=0$ under the operator-induced transition law.
\end{proposition}
\begin{proof}
Let $J(\pi_k)=k/(n-1)$ for $k=0,\ldots,n-1$ under the uniform policy
distribution, and let $J_V$ swap the values of adjacent policies
$\pi_k,\pi_{k+1}$ for any fixed $k\in\{0,\ldots,n-2\}$. Then
$\mathrm{MAE}=2/[n(n-1)]$ and $1-\rho_S=12/[n(n^2-1)]$, both tending to
zero. If the operator always proposes $(\pi_{k+1},\pi_k)$, then
$\Delta_V>0$ while $\Delta_*<0$, giving $\mathrm{IRR}=1$ and
$\mathrm{ISC}=0$.
\end{proof}

\paragraph{Operator shift and deployment response.}
\begin{proposition}[Operator shift bound]\label{prop:operator-shift}
Let $P$ be a calibration transition law and $Q_{\mathcal A}\ll P$ with
$\chi^2(Q_{\mathcal A}\Vert P)<\infty$. For
$w=dQ_{\mathcal A}/dP$ and $\ell(\tau)=L(\Delta_V,\Delta_*)$ with finite
variance under $P$,
\begin{equation}
  \left|
    \E_{Q_{\mathcal A}}[\ell]-\E_P[\ell]
  \right|
  \leq
  \sqrt{
    \operatorname{Var}_P(\ell)\,
    \chi^2(Q_{\mathcal A}\Vert P)
  }.
  \label{eq:operator-shift}
\end{equation}
\end{proposition}
\begin{proof}
Since $\E_P[w]=1$,
$\E_{Q_{\mathcal A}}[\ell]-\E_P[\ell]
=\E_P[(w-1)(\ell-\E_P[\ell])]$.
Cauchy--Schwarz and
$\E_P[(w-1)^2]=\chi^2(Q_{\mathcal A}\Vert P)$ give the result.
\end{proof}

Deployment response is a distinct source of error:
$\Delta_{\mathrm{actor}}-\Delta_{\mathrm{direct}}
=J(\pip;\M[\pip])-J(\pip;\M[\pii])$.
The actor--direct gap therefore isolates the change in candidate value caused
by the world it induces; the strategic level additionally captures other-agent
response through $R_{-i}$.

\paragraph{When fidelity errors change decisions.}
\begin{proposition}[Decision preservation]\label{prop:decision-preservation}
Let $\varepsilon_C=\max_{j\in C^+}|\Delta_{V,j}-\Delta_{*,j}|$,
$j^*\in\arg\max_{j\in C^+}\Delta_{*,j}$, and
$\hat j_V\in\arg\max_{j\in C^+}\Delta_{V,j}$. If $j^*$ is unique, let
$m=\Delta_{*,j^*}-\max_{k\neq j^*}\Delta_{*,k}$. Then
\begin{equation}
  \mathrm{ISR}(\hat j_V)\leq2\varepsilon_C,
  \qquad
  m>2\varepsilon_C\ \Longrightarrow\ \hat j_V=j^*.
  \label{eq:decision-preservation}
\end{equation}
\end{proposition}
\begin{proof}
Optimality of $\hat j_V$ gives
$\Delta_{*,j^*}\leq\Delta_{V,j^*}+\varepsilon_C
\leq\Delta_{V,\hat j_V}+\varepsilon_C
\leq\Delta_{*,\hat j_V}+2\varepsilon_C$.
The same error bound preserves the top-1 ordering when $m>2\varepsilon_C$.
\end{proof}

With matched randomness, positive incumbent--candidate covariance reduces the
variance of their paired difference relative to independent evaluation with
the same marginal variances. Uniform policy fidelity is sufficient to control
update error, but global average accuracy need not protect operator-induced
updates. Operator shift and deployment response create distinct sources of
update-level error, and these errors change replacement decisions only
relative to candidate margins. This motivates directing scarce HF evidence to
uncertainty that can change the selected update.

\section{PIVOT-KG: Decision-Aware Validation}

PIVOT is a paired validation framework for candidate transitions; PIVOT-KG
is its decision-aware instantiation using the knowledge-gradient acquisition
rule below. We therefore allocate expensive high-fidelity (HF) evaluation to
uncertainty that can change the replacement decision. Given an incumbent and
$K$ candidate updates, the cheap verifier provides proxy improvements
$\{\Delta_{V,j}\}_{j=1}^K$. PIVOT-KG models the proxy-to-deployment correction
$G_j=\Delta_{*,j}-\Delta_{V,j}$.

For tractable sequential updating, we use a calibrated working Gaussian correction
model over candidate-wise vectors $G,\Delta_V,\Delta_*\in\R^K$:
\[
  G\sim\mathcal N(\mu,\Sigma_G),
  \qquad
  \Delta_*\mid\Delta_V
  \sim\mathcal N(\Delta_V+\mu,\Sigma_G).
\]
The prior parameters and observation variances are estimated from a disjoint
calibration set of paired proxy/HF transitions, so each held-out decision root
starts from this calibrated belief without using its HF outcomes.

A query of candidate $j$ produces a prescribed paired HF evaluation
$Y_j=\Delta_{*,j}+\eta_j$ with $\eta_j\sim\mathcal N(0,s_j^2)$, where $s_j^2$
is estimated and regularized from calibration data. Gaussian conditioning
uses a working likelihood with conditionally independent query noise.
The incumbent and candidate are evaluated
under matched initial state, exogenous randomness, and opponent initialization.
Each query updates the joint posterior by Gaussian conditioning; non-diagonal
$\Sigma_G$ lets evidence update correlated unqueried candidates. In the
fixed-budget procedure, each candidate is queried at most once.

Given queried data $D$, let $a(D)$ denote the posterior replacement decision
and $\mathcal R(D)$ its Bayes selection regret. Using $C^+$ defined above,
\[
\begin{aligned}
  a(D)
  &=\arg\max_{j\in C^+}\E[\Delta_{*,j}\mid D],\\
  \mathcal R(D)
  &=\E\!\left[
      \max_{j\in C^+}\Delta_{*,j}
      -\Delta_{*,a(D)}
      \,\middle|\,D
    \right].
\end{aligned}
\]
For a hypothetical HF observation $Y_j$, define its expected value of sample
information (EVSI) and cost-normalized acquisition score as
\begin{equation}
\begin{aligned}
  \mathrm{EVSI}_j
  &=\mathcal R(D)
    -\E_{Y_j\mid D}\!\left[\mathcal R\bigl(D\cup\{(j,Y_j)\}\bigr)\right],\\
  A_j&=\frac{\mathrm{EVSI}_j}{c_j}.
\end{aligned}
  \label{eq:acquisition}
\end{equation}
PIVOT-KG therefore queries the feasible candidate with the largest expected
reduction in selection regret per unit HF cost. Implementations use Monte Carlo
or analytic Gaussian integration for these posterior expectations.

\begin{algorithm}[H]
\caption{PIVOT-KG under a fixed high-fidelity budget}
\label{alg:pivot-kg}
\begin{algorithmic}[1]
\REQUIRE Incumbent $\pi$; candidates $\{\pi'_j\}_{j=1}^K$;
calibrated model $(\mu,\Sigma_G,\{s_j^2\}_{j=1}^K)$;
query costs $\{c_j\}_{j=1}^K$; HF budget $B$
\ENSURE Selected candidate or incumbent
\STATE Evaluate candidates to obtain $\{\Delta_{V,j}\}_{j=1}^K$
\STATE Initialize the calibrated posterior,
$D\leftarrow\emptyset$, $S\leftarrow\emptyset$, and $b\leftarrow B$
\WHILE{there exists $j\in C\setminus S$ with $c_j\leq b$}
  \STATE Compute $A_j$ for each feasible $j\in C\setminus S$
  \STATE
  $j_q\leftarrow
  \arg\max_{j\in C\setminus S:\,c_j\leq b} A_j$
  \STATE Obtain paired HF observation $Y_{j_q}$ under matched randomness
  \STATE $D\leftarrow D\cup\{(j_q,Y_{j_q})\}$ and condition the posterior
  \STATE $S\leftarrow S\cup\{j_q\}$ and $b\leftarrow b-c_{j_q}$
\ENDWHILE
\RETURN $a(D)$
\end{algorithmic}
\end{algorithm}

The fixed-budget rule is our primary procedure; posterior-based early stopping
is a secondary variant.

\section{Experiments}

The experiments follow the progression from update-level fidelity failure to
its decision consequences and budgeted validation. We first isolate operator
shift and deployment response, then test when the resulting errors alter the
replacement decision, and finally evaluate whether PIVOT-KG can reduce
selection regret with limited high-fidelity (HF) evidence. We conclude with a
reactive driving stress test.

\subsection{Experimental Protocol}
\label{sec:exp-protocol}

All experiments are conducted entirely in simulation.
Controlled response studies use Performative Control and Congestion Resource
to isolate operator shift and deployment response; the pooled operator-shift
contrast additionally includes a frozen MPE2~\citep{lowe2017maddpg} condition
described in Appendix~\ref{app:fidelity-response}. Multi-agent evaluations use Leduc and
Kuhn poker through OpenSpiel~\citep{lanctot2019openspiel} and the Melting Pot
suite~\citep{leibo2021meltingpot}. Controlled studies use paired
common-random-number rollouts with 30 independent seeds, and the held-out
multi-agent study contains 30 roots in each environment. Each held-out decision
root has a fixed panel of candidate policy updates generated by its improvement
operator; selection also permits the incumbent no-update option. Held-out comparisons
use paired root-bootstrap confidence intervals (CIs) conditional on the fitted calibration
model; the controlled-study aggregation is specified in
Appendix~\ref{app:fidelity-response}. An HF query is one prescribed paired
candidate evaluation under matched randomness. Primary comparisons use
PIVOT-KG and Uniform; controlled batch diagnostics additionally include
Proxy Only, value-of-information (VOI) and lower and upper confidence bound (LUCB) heuristics, and an all-HF reference. HighwayEnv~\citep{highwayenv2018} is the
primary reactive driving stress test. A second Leduc cohort and a secondary
MetaDrive stress test are reported in the appendix. Bootstrap intervals
resample the corresponding inferential unit; further reproducibility details
are provided in the supplement.

\subsection{Improvement Fidelity under Operator Shift and Deployment Response}
\label{sec:exp-fidelity}

We first test the two mechanisms identified by the theory. Operator shift
changes which proposed transitions matter, whereas deployment response changes
the world in which a candidate is evaluated. In the operator-shift study, the
pooled high-minus-low contrast in operator-relative IDE is
\OperatorShiftEffect{} \OperatorShiftEffectCI{} (descriptive cell-bootstrap
interval, \OperatorShiftSeeds{} seed values; Appendix~\ref{app:fidelity-response}).
Local update error and reversals increase with the designed shift intensity,
while global policy rank is a fixed calibration reference
(Figure~\ref{fig:operator-shift}). Thus global policy accuracy can mask
failure on the updates an improvement operator actually proposes.

\begin{figure}[!b]
  \centering
  \includegraphics[width=\linewidth]{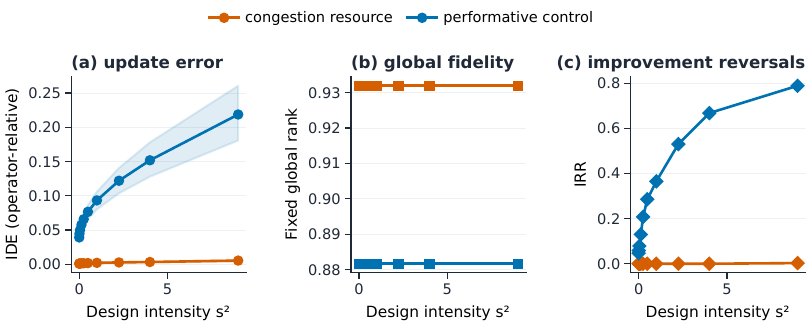}
  \caption{Operator shift exposes update-level failure. The horizontal axis
  is the design intensity $s^2$, not a measured $\chi^2$ divergence. Global
  Spearman rank is a fixed calibration reference.}
  \label{fig:operator-shift}
\end{figure}

Deployment response provides a distinct source of discrepancy. In the
strategic fixture, the strategic-minus-actor effect is
\StrategicEffect{} \StrategicEffectCI{}, with family-balanced strategic
reversal rate \StrategicReversalRate{}. A second Leduc confirmation cohort shows the
same mechanism along the response horizon: the proxy-positive reversal rate
rises from 2.5\% under one-step response to 51.7\% under eight-step adaptation
(Appendix~\ref{app:fidelity-response}). The appendix also compares
response-disabled proxy, actor, and strategic evaluations across opponents. Together,
these results identify two distinct sources of update-level failure: operator
shift can concentrate error on the transitions proposed by the improvement
process, while deployment response can sharply alter the value and sign of a
candidate after deployment.

\subsection{Decision Relevance and High-Fidelity Resolvability}
\label{sec:exp-decision}

Proposition~\ref{prop:decision-preservation} shows that update error changes a
replacement decision only when it is large relative to the candidate margin.
For each held-out root, we compare the deployment winner with its competitors
using a winner-relative correction-to-margin ratio; values above one indicate
disjoint proxy and deployment optimal sets, with one as the tie boundary.
Appendix~\ref{app:pivot-ext} gives the exact definition.

Across \RevisionDecisionRoots{} held-out roots, proxy and deployment optimal sets
are disjoint in \RevisionDecisionFlips{} cases. Figure~\ref{fig:decision-relevance}
shows that disagreement alone does not determine validation value: what matters
is whether the fidelity error is large enough to threaten the replacement
decision and whether available HF evidence can resolve the resulting margin.
At the primary long budget, the PIVOT query set intersects the ex-post
decision-critical pair in \RevisionCriticalQueryRate{}
\RevisionCriticalQueryRateCI{} of roots, while the mean model-assumed queried
observation standard deviation (SD) is \RevisionCriticalNoiseMargin{} \RevisionCriticalNoiseMarginCI{}
times the proxy top-2 margin. This regularized-model diagnostic does not
estimate HF observation noise directly; targeting alone need not resolve ordering.

\begin{figure}[!t]
  \centering
  \includegraphics[width=\linewidth]{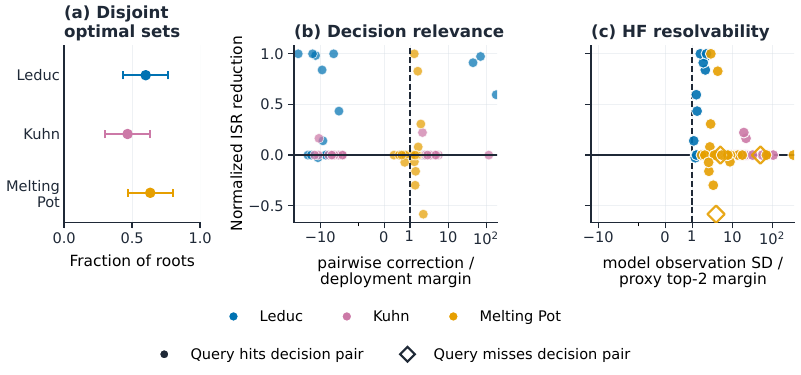}
  \caption{Decision relevance and HF resolvability. (a) Disjoint optimal sets.
  (b) Winner-relative correction-to-margin ratio versus
  normalized ISR reduction. (c) Model-assumed observation SD relative to the
  proxy top-2 margin versus normalized ISR reduction; filled markers hit the ex-post
  decision-critical pair.}
  \label{fig:decision-relevance}
\end{figure}

The held-out analysis isolates the mechanism behind useful validation:
fidelity error must threaten the replacement decision, and the available HF
evidence must be informative enough to resolve that uncertainty.

\subsection{PIVOT-KG under Limited High-Fidelity Budgets}
\label{sec:exp-pivot}

Table~\ref{tab:latest-main} reports the primary held-out fixed-budget comparison.
PIVOT-KG substantially lowers ISR in Leduc, the contrast is small in Kuhn, and
the primary Melting Pot comparison is unresolved. This heterogeneity is
consistent with the preceding mechanism analysis: deployment response alone
does not determine the value of targeted validation. Secondary Melting Pot
conditions are reported in Appendix~\ref{app:pivot-ext}.
For these cohorts, PIVOT-KG and Uniform share the calibrated posterior,
observation-update rule, and final selection rule, differing only in query
allocation; related component diagnostics are reported in
Appendix~\ref{app:pivot-ext}.

\begin{table}[!ht]
  \centering
  \footnotesize
  \caption{Primary held-out fixed-budget comparisons (30 paired roots per row).
  Positive Uniform $-$ PIVOT means lower ISR for PIVOT-KG; brackets give 95\%
  root-bootstrap CIs.}
  \label{tab:latest-main}
\begin{tabular*}{\linewidth}{@{\extracolsep{\fill}}lcrrl@{}}
\toprule
World / response & HF queries & Uniform ISR & PIVOT-KG ISR & \shortstack[l]{Uniform $-$ PIVOT\\{[95\% CI]}} \\
\midrule
Leduc, long & 2 & 0.046 & 0.009 & 0.036 [0.014, 0.064] \\
Kuhn, long & 2 & 0.058 & 0.056 & 0.002 [0.000, 0.006] \\
Melting Pot, long & 2 & 1.742 & 0.836 & 0.906 [-0.784, 3.401] \\
\bottomrule
\end{tabular*}
\end{table}

We also retain controlled batch-selector diagnostics, distinct from
Algorithm~\ref{alg:pivot-kg}. Their pooled candidate-only cumulative improvement-selection regret (CISR) contrast for Proxy-Only minus PIVOT-batch under the legacy forced-replacement setting is \ClosedLoopEffect{} \ClosedLoopEffectCI{}
over \ClosedLoopRounds{} rounds; Figure~\ref{fig:efficiency} compares the
single-decision batch selectors at matched query counts. A second Leduc
confirmation cohort yields a prespecified PIVOT-KG versus expected-Uniform
selected-gain contrast of +0.02989 using a repeated-query, sample-size variant.
Its component comparisons, the legacy closed-loop target, stopping behavior,
and numerical diagnostics are specified in Appendix~\ref{app:pivot-ext}.

\begin{figure}[!t]
  \centering
  \includegraphics[width=\linewidth]{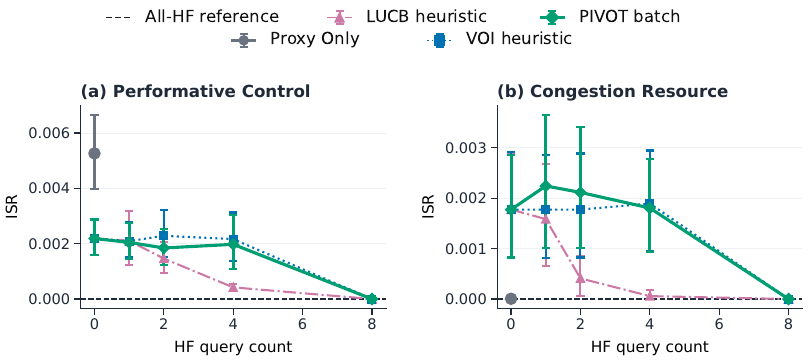}
  \caption{Legacy batch-selector diagnostics at matched HF query counts in
  (a) Performative Control and (b) Congestion Resource. Lower single-decision
  ISR is better; the all-HF reference uses only fully queried rows.
  Intervals bootstrap seeds.}
  \label{fig:efficiency}
\end{figure}

Together, these results support the central prediction of our analysis:
decision-aware validation is most useful when the remaining uncertainty is
both consequential for the replacement decision and resolvable by the
available HF evidence.

\subsection{Reactive Driving Stress Test}
\label{sec:exp-highway}

We finally test PIVOT-KG in a reactive driving simulation where HF evaluation
is scarce relative to the candidate set. The HighwayEnv redesign uses eight
scheduled candidates formed by crossing four non-idle actions with early and
late half-horizon windows, fits the correction model on 40 calibration roots,
and evaluates 120 disjoint test roots. At the primary budget $B=4$, PIVOT-KG
reduces ISR relative to the exact Uniform validation rule by
\HighwayRedesignBudgetFourContrast{} \HighwayRedesignBudgetFourCI{}.
The Uniform rule averages exactly over all feasible queried subsets, and
each PIVOT query corresponds to a paired actor rollout in the reactive simulator.

The HighwayEnv result shows the same decision-aware validation principle in
a reactive driving candidate-allocation problem. Earlier HighwayEnv
cohorts were unresolved at their primary budgets; these results, the secondary
$B=2$ redesign result, and a secondary
MetaDrive~\citep{li2022metadrive} stress test are reported in
Appendix~\ref{app:physical-stress}.

\section{Conclusion}

This paper reframes reliable self-improvement as an operator-relative,
update-level evaluation problem. Improvement Fidelity asks whether a proposed
improvement under a proxy remains an improvement after deployment in the world
that the update
induces. Our analysis shows why strong policy-level evaluation is not
sufficient: operator shift and deployment response can create update-level
errors, while candidate margins determine whether those errors change the
replacement decision. PIVOT-KG turns this structure into a decision-aware
validation rule that allocates scarce high-fidelity evidence toward
uncertainty that can affect the selected update. Across controlled,
multi-agent, and reactive driving simulations, our experiments support this
view and show when targeted validation can reduce selection regret.

The current study has three main limitations. First, the direct, actor, and
strategic response models capture several important forms of adaptation but do
not cover all deployment dynamics. Second, PIVOT-KG relies on a calibrated
correction model and sufficiently informative HF observations; when evidence
is too noisy or budgets are too small, decision-critical uncertainty can
remain unresolved. Reported held-out intervals condition on the fitted calibration
model and do not propagate calibration-set sampling uncertainty. Third, our
driving-simulator evidence has its clearest allocation result in HighwayEnv
and a more mixed MetaDrive stress test.

These results suggest several directions for future work, including richer
response and correction models, repeated self-improvement under co-adaptation,
and intervention-based HF evaluation in more realistic deployment systems. More
broadly, as agents increasingly modify their own policies, the central
evaluation question is not only whether a candidate scores better under a
verifier, but whether replacing the incumbent with that candidate remains an
improvement in the world the replacement creates. When better can become worse
after deployment, reliable self-improvement requires fidelity of improvements,
not merely accuracy of policy evaluations---a principle that provides both a
theoretical criterion for update reliability and a practical basis for allocating
scarce deployment evidence when it can affect the replacement decision.

\paragraph{AI assistance.}
The authors developed the research and initial manuscript draft. Generative AI
tools assisted limited code implementation, scientific consistency checks,
and final language and presentation polishing. The authors take full
responsibility for the research, experimental results, and manuscript.

\label{refs:start}

\clearpage
\appendix
\raggedbottom

\section{Response Mechanism Analyses}
\label{app:fidelity-response}

The operator contrast pools Performative Control, Congestion Resource, and
frozen Multi-Agent Particle Environment 2 (MPE2)~\citep{lowe2017maddpg}
with weights $0.4/0.4/0.2$; Figure~\ref{fig:operator-shift} shows the
first two. Its interval resamples 450 condition/seed cells sharing 30 seeds;
figure bands resample six condition means. Congestion strengths are duplicated;
global rank is a fixed zero-shift evolutionary-policy reference. These are
descriptive shift diagnostics, not measured $\chi^2$ divergence.
Figure~\ref{fig:layers}'s proxy disables response, unlike the fixed-incumbent-world
$\Delta_{\mathrm{direct}}$. Strategic reversal averages the three adaptive families.

\begin{figure}[!htbp]
  \centering
  \includegraphics[width=\linewidth]{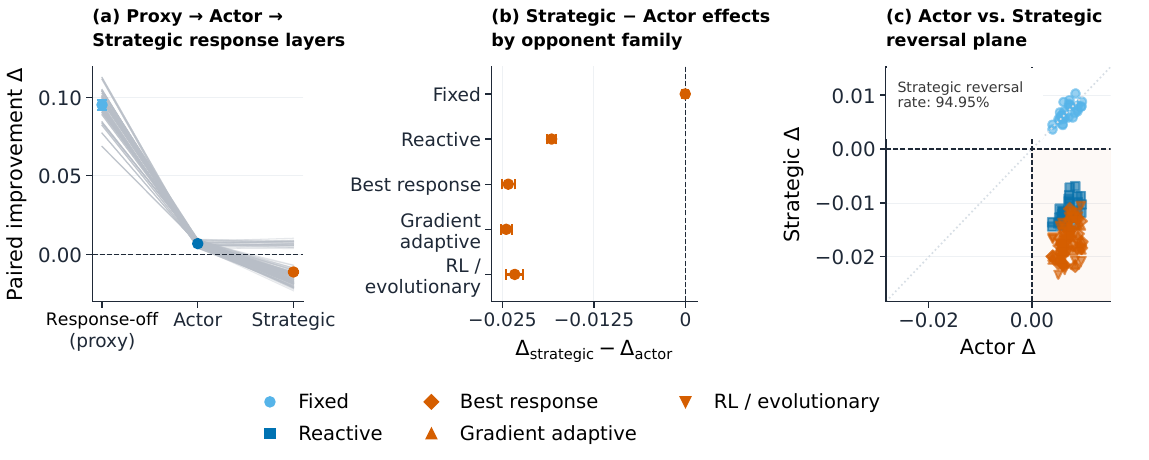}
  \caption{Response mechanisms. (a) Response-disabled proxy, actor, and
  strategic evaluations. (b) Strategic-minus-actor effects by opponent family. (c) Actor
  versus strategic improvement. Intervals use 30 matched seeds.}
  \label{fig:layers}
\end{figure}

A separate Leduc confirmation cohort tests the same mechanism along the
response horizon. Among proxy-positive candidates, the mean root-level
improvement-reversal rate rises from 2.5\% under one-step response to 51.7\%
under eight-step adaptation (Figure~\ref{fig:leduc-response-horizon}). This
supports the main-text conclusion that deployment response can amplify
update-level failure rather than merely shift policy values uniformly.

\begin{figure}[!htbp]
  \centering
  \includegraphics[width=0.70\linewidth]{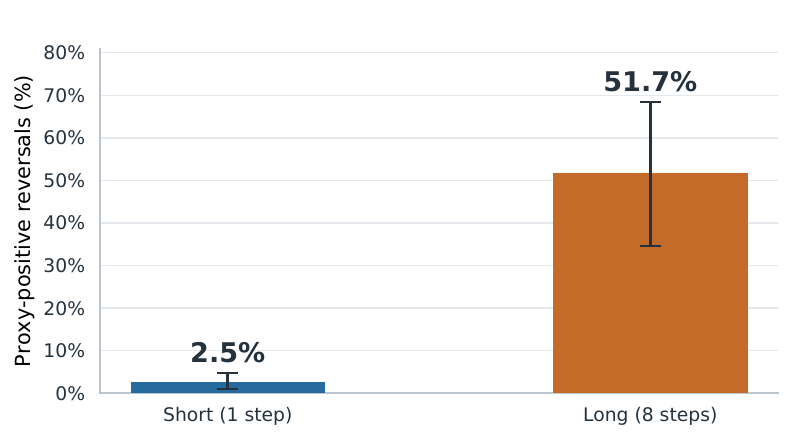}
  \caption{Leduc response-horizon confirmation. Reversal rates are computed
  among proxy-positive updates within each root, where a reversal means
  negative deployment gain. Bars show means across 30 roots under one-step
  and eight-step adaptation, with 95\% root-bootstrap CIs.}
  \label{fig:leduc-response-horizon}
\end{figure}

\FloatBarrier
\section{Extended PIVOT-KG Evaluation}
\label{app:pivot-ext}

In the calibrated poker, Melting Pot, and driving evaluations, candidate
coordinates follow the same fixed update slots across calibration and test roots.

\paragraph{Decision diagnostics.}
Let $T=\arg\max_{C^+}\Delta_*$ and $b\in\arg\max_{j\in T}\Delta_{V,j}$.
The ratio is $\max_{k\notin T}(G_b-G_k)/(\Delta_{*,b}-\Delta_{*,k})$,
with $G=\Delta_*-\Delta_V$ and tie tolerance $10^{-12}$.
The 51 disagreements are disjoint optimal sets; 28 roots have tied optima.
Above one indicates disjoint optimal sets; one is the tie boundary.
The query-hit pair uses proxy/deployment top candidates within $C$ (stable index ties);
top-2 margins exclude no-update. Normalization divides Uniform minus PIVOT ISR
by the deployment range over $C^+$. The noise ratio uses mean queried $\sqrt{R_j}$,
a margin floor of $10^{-12}$, and poker's model-variance floor $0.25$;
it is model-assumed, not directly measured observation noise.

At the primary long budget, the decision-critical query-hit rates are
\RevisionLeducCriticalQueryRate{} in Leduc,
\RevisionKuhnCriticalQueryRate{} in Kuhn, and
\RevisionMeltingCriticalQueryRate{} in Melting Pot. The descriptive normalized
gain is \RevisionCriticalQueryHitGain{} \RevisionCriticalQueryHitGainCI{}
among the 87 hit roots and
\RevisionCriticalQueryMissGain{} \RevisionCriticalQueryMissGainCI{} among the
three misses. These diagnostics separate successful targeting from successful
resolution.

\paragraph{Secondary and confirmation analyses.}
The first cohorts use 12 calibration/30 test roots and one frozen Uniform
draw. Poker's nine queryable entries include an incumbent copy; Melting Pot
uses eight specialist mixtures on one repeated stag-hunt substrate. Full
candidate/cost protocols accompany the supplement. Designated allocation
interactions are $0.0365\,[0.0145,0.0637]$ (Leduc), $0.0020\,[0,0.0058]$
(Kuhn), and $1.4815\,[-0.3930,4.0956]$ (Melting Pot); only Leduc meets the
joint primary/interaction criterion.

Table~\ref{tab:secondary-sealed} retains the secondary comparisons. The second
Leduc cohort was designed after the first-cohort results, with fresh confirmation
roots. This cohort is a mechanism confirmation study rather than a direct replication of the fixed-budget benchmark. Its PIVOT-KG variant queries $(j,n)$ pairs with $n\in\{512,2048,4096\}$,
permits repeated candidates, and charges $2n$ hands; it extends the once-per-candidate
Algorithm~\ref{alg:pivot-kg}. At 16,384 hands its prespecified contrast against
exact-subset expected Uniform is +0.02989 under long adaptation.
Figure~\ref{fig:leduc-components} retains the component comparisons; its final
row is $(\mathrm{PIVOT}-\mathrm{Uniform})_{\rm long}-(\mathrm{PIVOT}-\mathrm{Uniform})_{\rm short}$,
not a direct long-minus-short policy-gain contrast.

\begin{table}[!htbp]
  \centering
  \footnotesize
  \caption{Secondary held-out comparisons. Positive Uniform $-$ PIVOT means
  lower ISR for PIVOT-KG; brackets give 95\% root-bootstrap CIs.}
  \label{tab:secondary-sealed}
\begin{tabular*}{\linewidth}{@{\extracolsep{\fill}}lcrrl@{}}
\toprule
World / response & HF queries & Uniform ISR & PIVOT-KG ISR & \shortstack[l]{Uniform $-$ PIVOT\\{[95\% CI]}} \\
\midrule
Melting Pot, long & 1 & 5.283 & 1.859 & 3.423 [0.390, 6.944] \\
Melting Pot, short & 5 & 1.035 & 1.610 & -0.575 [-1.042, -0.180] \\
\bottomrule
\end{tabular*}
\end{table}

\begin{figure}[H]
  \centering
  \includegraphics[width=0.90\linewidth]{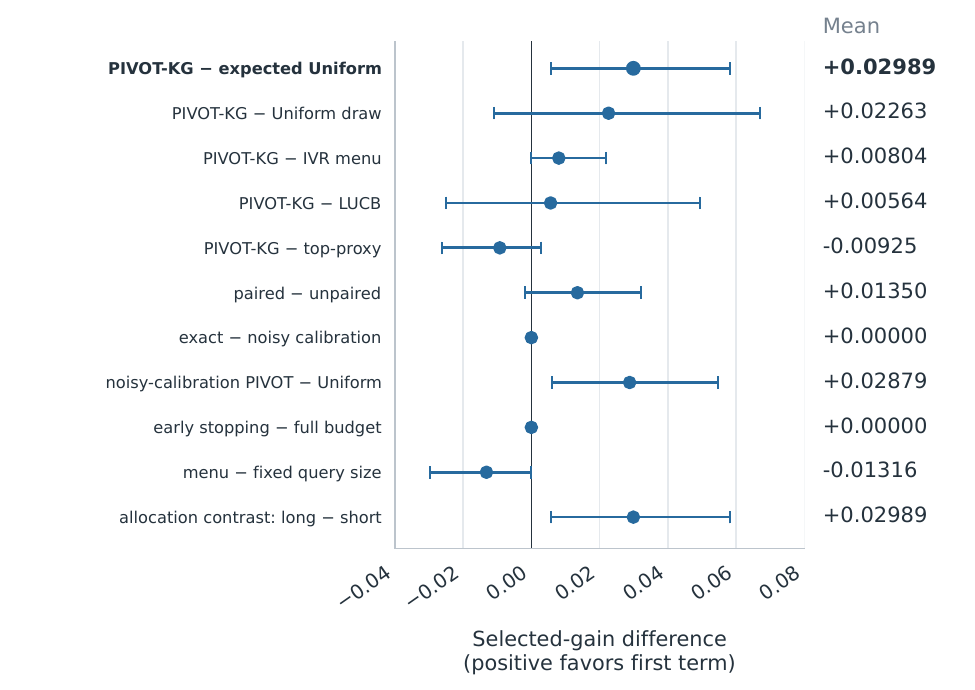}
  \caption{Leduc confirmation-cohort component contrasts under long response
  (30 roots; primary budget 16,384 HF hands). Positive values favor the first
  term. The expected-Uniform comparison is prespecified primary; all others
  are secondary with unadjusted 95\% CIs.}
  \label{fig:leduc-components}
\end{figure}

The secondary stopping rule reduces mean HF use from 16,384 to 12,083.2 hands,
a 26.25\% reduction, while producing the same selected gain as the full-budget
rule on all 30 observed roots. The paired difference in HF use is $-4300.8$
$[-6007.5,-2662.4]$ hands. Thus the stopping variant can save validation
queries in this cohort without changing the observed replacement decision.

\paragraph{Closed-loop and numerical diagnostics.}
Figures~\ref{fig:efficiency}, \ref{fig:closed-loop}, and~\ref{fig:robustness}
use legacy batch selection, not Algorithm~\ref{alg:pivot-kg}: no within-batch
reacquisition; observed values replace queried estimates; unqueried estimates
can omit later observations. The closed loop forces replacement and reports
$\mathrm{CISR}_C=\sum_t[\max_{j\in C_t}\Delta_{*,j}-\Delta_{*,\hat j_t}]$
without no-update; negative improvement with zero regret is then possible.
Figure~\ref{fig:efficiency} counts queries and uses VOI/LUCB heuristics.
Figure~\ref{fig:robustness} compares Monte Carlo draws on 24 Congestion test
groups ($K=4$): all costs scale equally, query count stays two, and seeds vary.
With common draws the ranking is invariant, so this does not identify cost misspecification.

\begin{figure}[!htb]
  \centering
  \includegraphics[width=\linewidth]{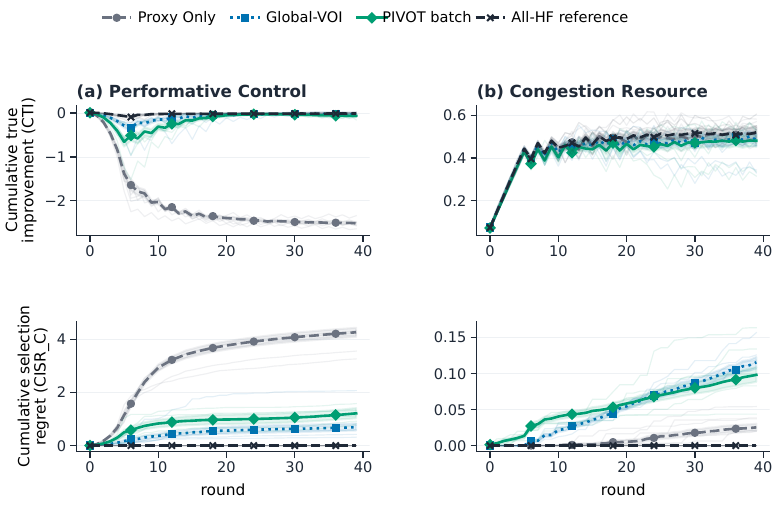}
  \caption{Legacy forced-replacement trajectories. Top: cumulative true
  improvement; bottom: candidate-only $\mathrm{CISR}_C$, excluding no-update.
  Thin curves are seeds; markers summarize batch decision rules.}
  \label{fig:closed-loop}
\end{figure}

\begin{figure}[!htb]
  \centering
  \includegraphics[width=\linewidth]{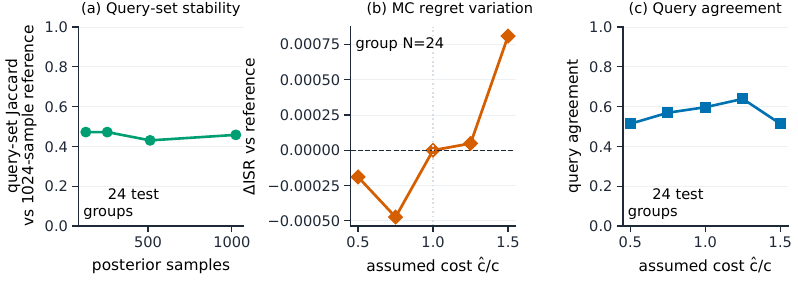}
  \caption{Legacy batch-selector Monte Carlo diagnostics on 24 test groups.
  Panels show query agreement across sample counts and regret/query variation
  under common cost scaling with different random draws; this is not an
  identified cost-misspecification effect.}
  \label{fig:robustness}
\end{figure}

Additional response-mixture settings yield positive, null, and negative
allocation effects, consistent with the main-text conclusion that deployment
response alone does not determine validation value. Full method-by-method and
response-mixture diagnostics are provided in the supplement.

\section{Reactive Driving Stress Tests}
\label{app:physical-stress}
\label{app:highway-replication}

\paragraph{HighwayEnv.}
The original and fresh-seed replication cohorts each use 20 calibration and 60
test roots, with no overlap between cohorts. Both use the same five-action
panel; the primary budget is $B=2$ and $B=4$ is secondary.
Intervals use paired root bootstrap and condition on the fitted calibration
model. A technical smoke run exposed two replication roots; they are
retained without subsequent tuning or exclusion. Uniform retains proxy
estimates for unqueried candidates, whereas PIVOT-KG uses its calibrated
posterior, so the comparison evaluates complete selection rules rather than
acquisition alone.

The eight-candidate redesign freezes eight scheduled candidates, fits a joint
correction model on 40 calibration roots, and evaluates 120 disjoint test
roots. Its primary budget is $B=4$ and $B=2$ is secondary. Uniform
is the exact average over all $\binom{8}{B}$ queried subsets (70 at $B=4$ and
28 at $B=2$), while PIVOT-KG applies sequential Gaussian updates. Each query is
a paired actor rollout in the reactive simulator. Two test roots were exposed
during technical smoke and are retained without retuning or exclusion. Detailed execution and
audit records are provided in the supplement.

\begin{table}[!htbp]
  \centering
  \footnotesize
  \caption{HighwayEnv cohorts. Positive Uniform $-$ PIVOT means lower ISR for
  PIVOT-KG; brackets give paired 95\% bootstrap CIs. Bold budgets mark the primary comparison within each cohort.}
  \label{tab:highway-budget}
\begin{tabular*}{\linewidth}{@{\extracolsep{\fill}}lcrrrl@{}}
\toprule
Cohort & $B$ & Roots & Uniform ISR & PIVOT-KG ISR & \shortstack[l]{Uniform $-$ PIVOT\\{[95\% CI]}} \\
\midrule
Original & 1 & 60 & 0.0492 & 0.0718 & $-0.0226$ $[-0.0439,-0.0008]$ \\
Original & \textbf{2} & 60 & 0.0422 & 0.0374 & $0.0049$ $[-0.0088,0.0199]$ \\
Original & 4 & 60 & 0.0144 & 0.0029 & $0.0115$ $[0.0008,0.0237]$ \\
\addlinespace[3pt]
Replication & \textbf{2} & 60 & 0.0254 & 0.0261 & $-0.0006$ $[-0.0143,0.0145]$ \\
Replication & 4 & 60 & 0.0037 & 0.0004 & $0.0034$ $[-0.0007,0.0099]$ \\
\addlinespace[3pt]
Redesign (8) & 2 & 120 & 0.0613 & 0.0157 & $0.0456$ $[0.0343,0.0587]$ \\
Redesign (8) & \textbf{4} & 120 & 0.0435 & 0.0055 & $0.0380$ $[0.0300,0.0474]$ \\
\bottomrule
\end{tabular*}
\end{table}

\begin{figure}[!htb]
  \centering
  \includegraphics[width=0.90\linewidth]{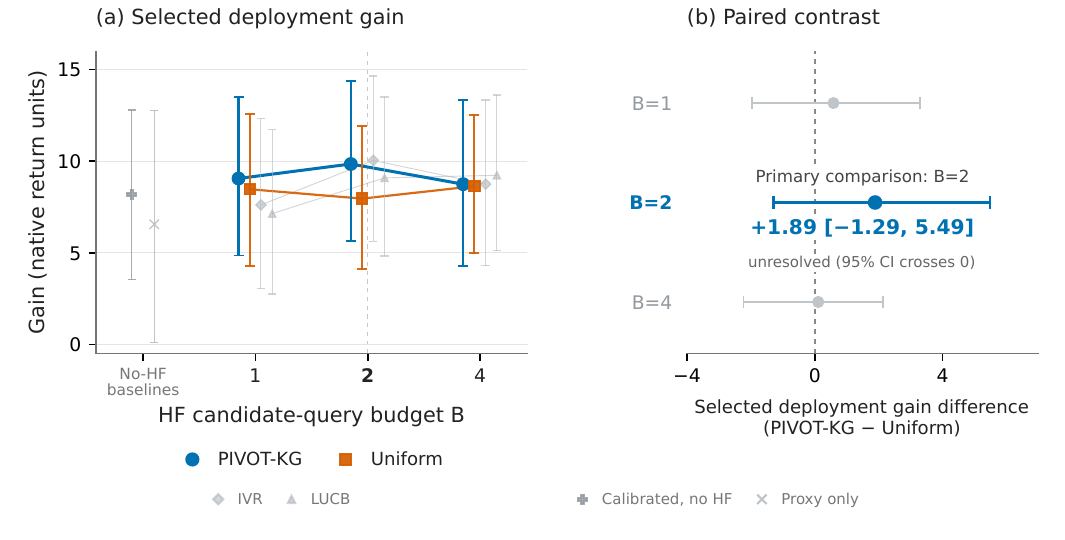}
  \caption{MetaDrive stress test as a boundary case. The prespecified $B=2$
  comparison remains unresolved: $+1.89\,[-1.29,5.49]$, with the 95\%
  root-bootstrap CI crossing zero. (a) Selected deployment gain (12 profiles,
  30 held-out roots; marginal 95\% root-bootstrap intervals).
  (b) Paired PIVOT-KG--Uniform selected deployment gain contrasts;
  $B=1,4$ are secondary.}
  \label{fig:metadrive-stress}
\end{figure}

\paragraph{MetaDrive.}
We additionally evaluate a first MetaDrive~\citep{li2022metadrive} cohort under
a long-response condition with 12 profiles and 30 held-out roots. Figure~\ref{fig:metadrive-stress} reports selected deployment gain across
one, two, and four HF candidate-query budgets. At the prespecified two-query
budget, the paired PIVOT-KG--Uniform contrast is
$+1.89\,[-1.29,5.49]$ (30-root bootstrap, 10,000 draws), leaving the primary
comparison unresolved. The $B=1$ and $B=4$ comparisons are secondary with
unadjusted intervals. We retain MetaDrive as boundary evidence rather than a
stable positive result; roots share one fixed bottleneck geometry.
Integrated variance reduction (IVR) and LUCB use the matched posterior.
The calibrated no-HF and proxy-only baselines incur zero HF cost and appear
as the separate No-HF baselines. Uniform averages 100 allocation draws per root.
Positive-query costs include response search and paired evaluation; the
supplement retains all eight rules.

\paragraph{Evaluator boundary.}
Improvement Fidelity defines an evaluation target, not a preferred learner
architecture. In an out-of-distribution comparison, a transition-specific
learner does not outperform the global learner: the transition-minus-global
ISC contrast is \OODGain{} \OODGainCI{}. Full ISC and IDE comparisons are
reported in the supplement.
\end{document}